\documentclass[conference]{IEEEtran}
\IEEEoverridecommandlockouts
\usepackage{cite}
\usepackage{amsmath,amssymb,amsfonts}
\usepackage{algorithmic}
\usepackage{graphicx}
\usepackage{textcomp}
\usepackage{xcolor}
\usepackage{amsthm}
\usepackage{balance}
\allowdisplaybreaks
\newtheorem{theorem}{Theorem}
\newtheorem{lemma}{Lemma}
\newtheorem{remark}{Remark}
\newtheorem{assumption}{\textbf{Assumption}}
\newtheorem{definition}{\textbf{Definition}}
\def\BibTeX{{\rm B\kern-.05em{\sc i\kern-.025em b}\kern-.08em
    T\kern-.1667em\lower.7ex\hbox{E}\kern-.125emX}}
\begin{document}

\title{Gradient Coding in Decentralized Learning for Evading Stragglers\\

\thanks{This work was supported by Digital Futures.}
}

\author{\IEEEauthorblockN{Chengxi Li}
\IEEEauthorblockA{\textit{School of Electrical Engineering and Computer Science} \\
\textit{KTH Royal Institute of Technology}\\
Stockholm, Sweden \\
chengxli@kth.se}
\and
\IEEEauthorblockN{Mikael Skoglund}
\IEEEauthorblockA{\textit{School of Electrical Engineering and Computer Science} \\
\textit{KTH Royal Institute of Technology}\\
Stockholm, Sweden \\
skoglund@kth.se}
}

\maketitle

\begin{abstract}
In this paper, we consider a \textit{decentralized learning} problem in the presence of stragglers. Although gradient coding techniques have been developed for \textit{distributed learning} to evade stragglers, where the devices send encoded gradients with redundant training data, it is difficult to apply those techniques directly to decentralized learning scenarios. To deal with this problem, we propose a new \underline{go}ssip-based decentralized learning method with gradient \underline{co}ding (GOCO). In the proposed method, to avoid the negative impact of stragglers, the parameter vectors are updated locally using encoded gradients based on the framework of stochastic gradient coding and then averaged in a gossip-based manner. We analyze the convergence performance of GOCO for strongly convex loss functions. And we also provide simulation results to demonstrate the superiority of the proposed method in terms of learning performance compared with the baseline methods.
\end{abstract}


\section{Introduction}
With increasing volume of data available for training models, the development and application of machine learning across various domains have been significantly boosted \cite{balkus2022survey}. In practice, training machine learning models has a resource-intensive nature. To address high demands for computational resources, edge learning is a vital approach offering several key benefits by distributing the computational tasks across multiple edge devices instead of relying on a single device \cite{feng2022heterogeneous}. Edge learning can be realized in two different ways, which are known as distributed learning \cite{liu2022distributed} and decentralized learning (DEL) \cite{beltran2023decentralized}. In distributed learning, a central server coordinates the training among the devices by maintaining a global model \cite{li2021communication, beznosikov2023biased}. In contrast, in DEL, the devices train a model collectively without the aid of a central server and they communicate with each other in a peer-to-peer manner \cite{taheri2023generalization, li2021decentralized}. Maintaining a central server as done in distributed learning results in very high infrastructure cost. In that case, DEL has gained much attention for its advantages of significant cost savings and enhanced scalability. 

In edge learning, some devices can be slow or unresponsive from time to time due to unanticipated incidents and they are known as stragglers \cite{yu2020straggler}. In the presence of stragglers, there is a notable degradation in learning performance. To mitigate their negative impact, gradient coding (GC) techniques have been proposed \cite{tandon2017gradient, ozfatura2019gradient, buyukates2022gradient, glasgow2021approximate,wang2019erasurehead,bitar2020stochastic,Chengxi20231-bitGC}. In GC, training data are distributed among devices redundantly. Then, during the training process, devices transmit encoded gradients computed from local training data and the global model update can be obtained robustly at the central server even in the presence of some stragglers. 

Existing GC techniques can be classified into exact GC techniques and approximate GC techniques. In exact GC techniques, the true model update can be recovered exactly as if no straggler was present. For instance, in \cite{tandon2017gradient}, each device encodes the local gradients by linear combinations with carefully-designed coefficients and sends an encoded vector in each iteration, where training data are allocated based on fractional repetition codes. In \cite{ozfatura2019gradient}, each device transmits multiple encoded vectors using local gradients to further reduce the training time. Noting that machine learning algorithms are robust to a certain level of noise, it is reasonable to only recover an approximate model update to update the trained model in each iteration. Based on that, approximate GC techniques are developed. In \cite{glasgow2021approximate}, approximate GC via expander graphs is proposed, which carefully designs the decoding coefficients to ensure optimal performance. In \cite{bitar2020stochastic}, the framework of stochastic gradient coding (SGC) is designed, where training data are allocated to the devices in a pair-wise balanced manner. Under this framework, it is proved that the convergence performance mirrors that of batched stochastic gradient descent (SGD). However, those methods are all proposed for distributed learning scenarios, which require to maintain a global model during the training process and can not be directly applied to cope with the DEL problem. 

In this paper, to deal with the DEL problem in the presence of stragglers, we propose a new \underline{go}ssip-based DEL method with gradient \underline{co}ding (GOCO) by combining the advantages of SGC in \cite{bitar2020stochastic} and gossip-based approaches. In the proposed method, multiple subsets of the training data are firstly distributed to the devices in a pair-wise balanced manner under the framework of SGC. Subsequently, during each iteration of the learning process, each nonstraggler device updates its parameter vector using encoded local gradients. It then averages the parameters of its own and those from its neighbors with a gossip-based approach, and transmits the averaged result to local neighbors as defined by the communication graph. We analyze the convergence performance of the proposed method for strongly convex loss functions. Finally, the superiority of the proposed method over the baseline methods is verified via simulation results. 

\section{System Model}
\label{system model}
With the training dataset containing \(m\) subsets \(\mathcal{D} = \left\{ {{\mathcal{D}_1},...,{\mathcal{D}_m}} \right\}\), the aim of DEL is to optimize the model parameter vector \({{\mathbf{x}}} \in \mathbb{R}^d\) to attain minimal training loss with $n$ edge devices. The optimal parameter vector can be expressed as \cite{koloskova2019decentralized, ding2023dsgd}
\begin{align}
\label{model1}
{{\mathbf{x}}^*} = \mathop {\arg {\text{ min}}}\limits_{{\mathbf{x}} \in {\mathbb{R}^d}} \left[ {f\left( {\mathbf{x}} \right) = \frac{1}{m}\sum\limits_{k = 1}^m {{f_k}\left( {\mathbf{x}} \right)} } \right],
\end{align}
where \(f_k({\mathbf{x}}): \mathbb{R}^d \to \mathbb{R}\) is the training loss associated with subset ${\mathcal{D}_k}$ \cite{koloskova2019decentralized, ding2023dsgd}:
\begin{align}
\label{fix}
{f_k}({\mathbf{x}}) = {\mathbb{E}_{{\xi _k}\sim{\mathcal{D}_k}}}{F_k}\left( {{\mathbf{x}},{\xi _k}} \right).
\end{align}
In (\ref{fix}), \({F_k}\left( {{\mathbf{x}},{\xi _k}} \right): \mathbb{R}^d \to \mathbb{R}\) is the training loss on sample \({{\xi _k}}\) in \({\mathcal{D}_k}\). 

In DEL scenarios, the training subsets are firstly distributed to the devices. Subsequently, in each iteration of the learning process, each device utilizes its local training data and information received from neighbors to update its parameter vector \cite{ding2023dsgd}, where neighbors are defined by a symmetric matrix \({\mathbf{W}} \in {\mathbb{R}^{n \times n}}\). In \({\mathbf{W}} \), as the \((i,j)\text{-th}\) element, \(w_{i,j}>0\) implies device \(i\) and device \(j\) are local neighbors and \(w_{i,j}=0\) indicates the opposite. Equivalently, the undirected communication graph is denoted as \(\mathcal{G}=([n],\mathcal{E})\), \([n]=\{{1,...,n}\)\}, where edge \({\{i,j\}}\in\mathcal{E}\) if and only if \(w_{i,j}>0, i,j \in [n]\). Here, we assume that \({\mathbf{W}} \) is doubly stochastic satisfying \({\mathbf{W1}} = {\mathbf{1}}\) and \({{\mathbf{1}}^T}{\mathbf{W}} = {{\mathbf{1}}^T}\), where \({\mathbf{1}}\) is an \(n \times 1\) all-one vector \cite{koloskova2019decentralized}. It is also assumed that \(w_{i,i}>0, i \in [n]\) \cite{koloskova2019decentralized}. We order the eigenvalues of \({\mathbf{W}} \) as  \(1 = \left| {{\lambda _1}\left( \mathbf{W} \right)} \right| > \left| {{\lambda _2}\left( \mathbf{W} \right)} \right| \geq ... \geq \left| {{\lambda _n}\left( \mathbf{W} \right)} \right|\). Representing the spectral norm of a matrix $\mathbf{X}$ as \({\left\|  \mathbf{X}  \right\|_2}\), we define $\rho  = 1 - \left| {{\lambda _2}\left( \mathbf{W} \right)} \right|$ and $\beta  = {\left\| {{\mathbf{I}} - {\mathbf{W}}} \right\|_2}$.

During the learning process, in each iteration, the probability of each device being a straggler is denoted as $p$. The straggler behavior is assumed to be independent among the devices and iterations \cite{bitar2020stochastic}. Based on that, to indicate their straggler behaviour, let us define
\begin{align}
\label{Iit}
I_i^t = 
\begin{cases} 
1, & \text{if device } i\text{ is not a straggler}, \\
0, & \text{if device } i\text{ is a straggler},
\end{cases}
\end{align}
in iteration $t$, $\forall t$. The random variables \(\left\{ {I_i^t,\forall i} \right\}\) in (\ref{Iit}) are independent and identically distributed Bernoulli random variables and their probability mass function can be expressed as 
\begin{align}
\label{Iit_bernoulli}
\Pr \left( {I_i^t = 1} \right) = 1 - p, \Pr \left( {I_i^t = 0} \right) = p.
\end{align}
In that case, in each iteration, only the nonstragglers perform computations and transmissions, which may degrade the learning performance of DEL if no countermeasure is adopted. Noting that existing GC techniques can only be applied when there is a global model maintained by a central server, they are not easily applicable for the DEL problem with stragglers \cite{tandon2017gradient, ozfatura2019gradient, buyukates2022gradient, glasgow2021approximate,wang2019erasurehead,bitar2020stochastic,Chengxi20231-bitGC}. To address this issue, we aim to propose a new DEL method to enhance the learning performance by evading the negative impact of stragglers. 

\section{The Proposed Method}
\label{the proposed method}

\begin{figure*}[h]
    \centering
    \includegraphics[width=0.85\textwidth]{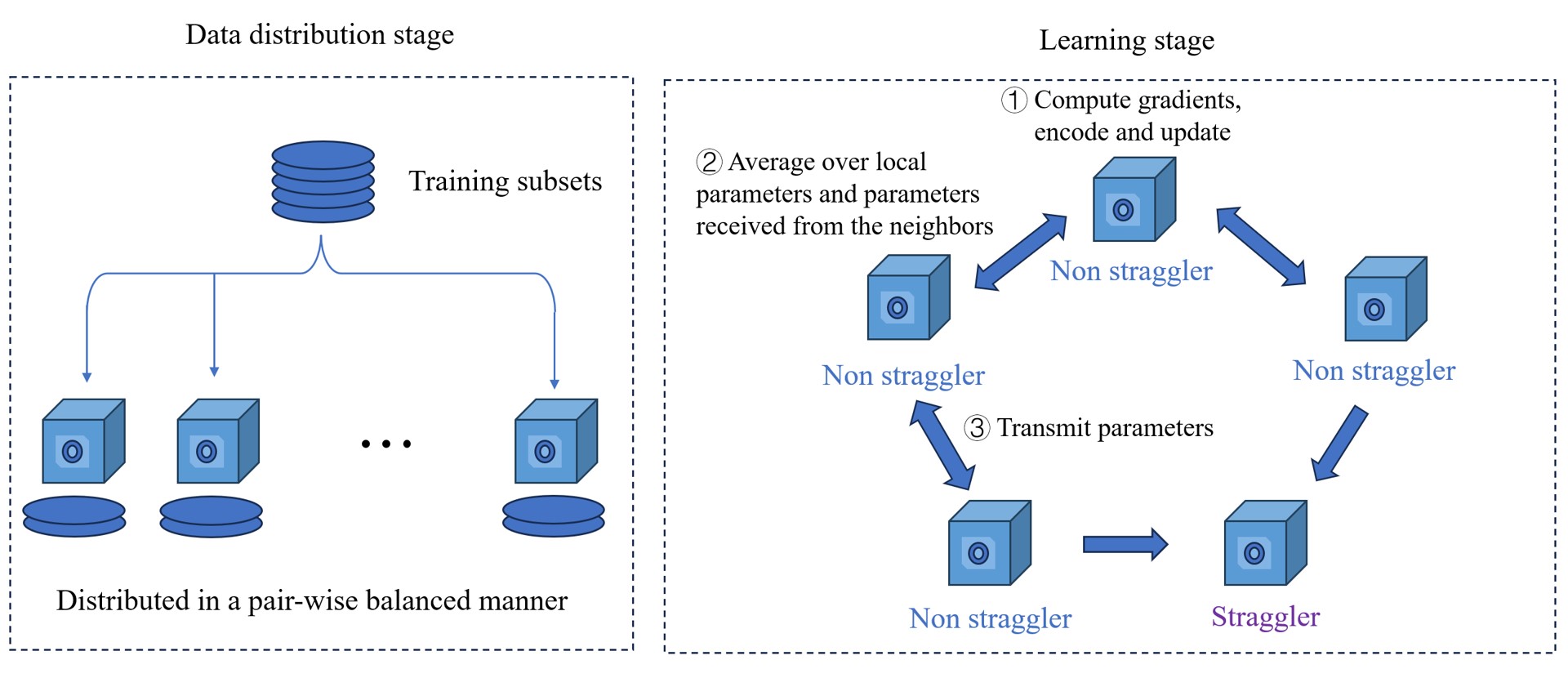}
    \caption{The flowchart of GOCO.}
    \label{fig:GOCO}
\end{figure*}

In this section, we propose our method to deal with the problem introduced in Section \ref{system model}. 

Motivated by the SGC framework introduced in \cite{bitar2020stochastic}, before the learning starts, the subsets \( \left\{ {{\mathcal{D}_1},...,{\mathcal{D}_m}} \right\}\) are distributed to the devices in a pair-wise balanced manner due to its advantages of easy implementation and ability to approximate a fully random distribution of training data. In this framework, subset \({\mathcal{D}_k}\) is allocated to \(d_k\) devices, \(\forall k\), and the number of devices that hold \({\mathcal{D}_{k_1}}\) and \({\mathcal{D}_{k_2}}\) together is given as \(\frac{{{d_{{k_1}}}{d_{{k_2}}}}}{n}\), for \({k_1} \ne {k_2}\). A matrix \({\mathbf{S}}\) is defined to indicate how the training subsets are distributed among the devices, where the \((i,k){\text{ - th}}\) element \({s_{i,k}}=1\) shows subset \({\mathcal{D}_k}\) is distributed to device \(i\) and \({s_{i,k}}=0\) indicates the opposite, \(i \in \left[ n \right]\), \(k \in \left[ m \right]\). 

At the beginning of the learning process, the parameter vector at device $i$ is initialized as \({{\mathbf{x}}_i^0}\), \(i \in \left[ n \right]\). Subsequently, in iteration $t$, for a nonstraggler device $i$, it samples \(\xi_{i,k}^{t}\) from \({\mathcal{D}_k}\) to calculate the stochastic gradient \( \nabla F_k(\mathbf{x}_i^{t}, \xi_{i,k}^{t})\) for \(k \in \left\{ {\left. k \right|{s_{i,k}} = 1} \right\}\), where the parameter vector of device $i$ is denoted by \({{\mathbf{x}}_i^t}\), \(\forall i\). After finishing this, the locally computed gradients are encoded into \({\mathbf{g}}_i^t\) at device $i$:  
\begin{align}
\label{g}
{\mathbf{g}}_i^t=\sum\limits_{k = 1}^m {\frac{1}{{{d_k}}}\nabla {F_k}({\mathbf{x}}_i^t,\xi _{i,k}^t){s_{i,k}}}.
\end{align}
Then, device $i$ updates its parameter vector with the encoded vector given as (\ref{g}) as 
\begin{align}
\label{x1/2}
{\mathbf{x}}_i^{t + \frac{1}{2}} = {\mathbf{x}}_i^t - \eta {\mathbf{g}}_i^t,
\end{align}
where \(\eta\) denotes the learning rate. Next, device $i$ takes the average over ${\mathbf{x}}_i^{t + \frac{1}{2}}$ and the vectors received from local neighbors with a gossip-based approach to obtain the updated parameter vector of iteration $t$: 
\begin{align}
\label{gossip_goco}
\mathbf{x}_i^{t+1} = \mathbf{x}_i^{t+\frac{1}{2}} + \gamma \sum_{\{i,j\} \in \mathcal{E}} w_{ij} ({\mathbf{x}}_j^{t} - {\mathbf{x}}_i^{t}),
\end{align}
where \(\gamma\) denotes the step size. At the end of iteration $t$, device $i$ sends $\mathbf{x}_i^{t+1}$ to its neighbors. 

Straggler devices do not take part in the current iteration. As a result, their local parameter vectors are unchanged. They only receive messages from their neighbors and store them locally. Note that this saved information can be used later when the devices become nonstragglers. The proposed GOCO method is illustrated as Fig. \ref{fig:GOCO}. 

Although the SGC framework adopted in this paper is inspired by existing work, there are significant differences between our proposed GOCO method and the original distributed learning method based on SGC proposed in \cite{bitar2020stochastic}. Specifically, GOCO integrates the advantages of SGC with a gossip-based approach in DEL scenarios to circumvent stragglers without relying on a central server. In contrast, the original SGC methodology requires a central server to ensure convergence in the presence of stragglers. While the training data distribution is similar in both this paper and \cite{bitar2020stochastic}, the manner in which information is shared among devices differs markedly, leading to distinctly different convergence analyses. This distinction will be seen more apparently in the following section.
\section{Theoretical Analysis}
\label{theoretical analysis}
In this section, the convergence performance of GOCO is analyzed. Before that, some useful definitions and assumptions are given to aid the derivation of the main theorem. 

\begin{definition}
\label{def1}
    (Strongly convex). A \(\mu\)-strongly convex function \(f({\mathbf{x}})\) is a function that satisfies  
     \begin{align}
    f({\mathbf{x}}) \geq f({\mathbf{x'}}) + \left\langle {\nabla f({\mathbf{x'}}),{\mathbf{x}} - {\mathbf{x'}}} \right\rangle  + \frac{\mu }{2}{\left\| {{\mathbf{x}} - {\mathbf{x'}}} \right\|^2},\forall {\mathbf{x}},{\mathbf{x'}} \in \mathbb{R}^d.
    \end{align}
\end{definition}

\begin{definition}
    \label{def2}
    (Smooth). An \(L\)-smooth function \(f({\mathbf{x}})\) is a function that satisfies 
    \begin{align}
    f({\mathbf{x}}) \leq f({\mathbf{x'}}) + \left\langle {\nabla f({\mathbf{x'}}),{\mathbf{x}} - {\mathbf{x'}}} \right\rangle  + \frac{L }{2}{\left\| {{\mathbf{x}} - {\mathbf{x'}}} \right\|^2},\forall {\mathbf{x}},{\mathbf{x'}} \in \mathbb{R}^d.
    \end{align}
\end{definition}

\begin{assumption}
\label{assumption-loss-function}
Each function \(f_k({\mathbf{x}})\) is \(\mu\)-strongly convex and \(L\)-smooth with bounded variance, i.e.,
\begin{align}
\label{boundvar1}
&\mathbb{E}_{\xi_k \sim \mathcal{D}_k
} \left\| \nabla F_k(\mathbf{x}, \xi_k) - \nabla f_k(\mathbf{x}) \right\|^2 \leq \sigma_k^2, \quad \forall \mathbf{x} \in \mathbb{R}^d, k \in [m],\\
\label{boundvar2}
   & \mathbb{E}_{\xi_k \sim \mathcal{D}_k
} \left\| \nabla F_k(\mathbf{x}, \xi_k) \right\|^2 \leq G^2, \quad \forall \mathbf{x} \in \mathbb{R}^d, k \in [m].
\end{align}
\end{assumption}

\begin{assumption}
    \label{unbiased}
    The stochastic gradients are unbiased:
\begin{align}
\label{unbias}
{\mathbb{E}_{{\xi _k}\sim {\mathcal{D}_k}}}\left[ {\nabla {F_k}({\mathbf{x}},{\xi _k})} \right] = \nabla {f_k}({\mathbf{x}}),\;\;\;{\kern 1pt} \forall {\mathbf{x}} \in {\mathbb{R}^d},k \in [m].
\end{align}
\end{assumption}

\begin{assumption}
    \label{subsets variation}
    The gradients associated with different training subsets satify 
\begin{align}
    \label{assump3}
    \left\| {\nabla {f_k}({{\mathbf{x}}^*}) - \nabla f({{\mathbf{x}}^*})} \right\| \leq C,\forall k,
\end{align}
where ${\mathbf{x}}^*$ is the optimal parameter vector given as (\ref{model1}). 
\end{assumption}

Next, we provide four lemmas, which are helpful for the derivation of the main theorem.  

\begin{lemma}
\label{xt-yitaIG-x*}
For the proposed method, we can derive that
\begin{align}
    \label{a11_final}
  \mathbb{E}&\left[ {\left. {{{\left\| {{{{\mathbf{\bar x}}}^t} - \frac{\eta }{n}\sum\limits_{i = 1}^n {I_i^t} {\mathbf{g}}_i^t - {{\mathbf{x}}^*}} \right\|}^2}} \right|{\Im _{t - 1}}} \right] \nonumber \\
   \leq& {q_1}{\left\| {{{{\mathbf{\bar x}}}^t} - {{\mathbf{x}}^*}} \right\|^2} + {q_2}\left\| {{{{\mathbf{\bar X}}}^t} - {{\mathbf{X}}^t}} \right\|_F^2 \nonumber \\
 & + \frac{{4{\eta ^2}}}{{{n^2}}}\left( {p - {p^2}} \right){C^2}\sum\limits_{k = 1}^m {\left( {\frac{1}{{{d_k}}} - \frac{1}{n}} \right)}  \nonumber \\
 & + 2\eta \frac{m}{n} \left( {1 - p} \right)\left[ {f\left( {{{\mathbf{x}}^*}} \right) - f\left( {{{{\mathbf{\bar x}}}^t}} \right)} \right] + \frac{{{\eta ^2}}}{{{n}}}\left( {1 - p} \right)\sum\limits_{k = 1}^m {\frac{{\sigma _k^2}}{{{d_k}}}},
\end{align}
where \({\left\|  \cdot  \right\|_F}\) denotes the Frobenius Norm, \(\mathbb{E}\left( {\left.  \cdot  \right|{\Im _{t - 1}}} \right)\) is the expectation taken conditioned on the previous iterations \(\left\{ {0,...,t-1 } \right\}\), ${{\mathbf{X}}^{t}} = \left[ {{\mathbf{x}}_1^t,...,{\mathbf{x}}_n^t} \right]$, ${{\mathbf{\bar x}}^t} = \frac{1}{n}\sum\limits_{i = 1}^n {{\mathbf{x}}_i^t}$, ${{\mathbf{\bar X}}^{t}} = \left[ {{\mathbf{\bar x}}^t,...,{\mathbf{\bar x}}^t} \right]$, ${q_1} = 1 - {q_{1,1}}\eta  + {q_{1,2}}{\eta ^2}$, ${q_{1,1}} = \frac{\mu \left( {1 - p} \right){a_{\min }}}{2}$, ${q_{1,2}} = {\frac{{2}}{{{n^3}}}\left( {p - {p^2}} \right){m^2}{L^2} + \frac{{2}}{{{n^2}}}{{\left( {1 - p} \right)}^2}{m^2}{L^2}} + \frac{{4{L^2}\left( {1 - p} \right)p}}{{{n^2}}}\sum\limits_{k = 1}^m {\left( {\frac{1}{{{d_k}}} - \frac{1}{n}} \right)}$, ${q_2} = {q_{2,1}}{\eta ^2} + {q_{2,2}}\eta$, ${q_{2,1}} = 2\left( {1 - p} \right){L^2}\left( {\frac{{pa_{\max }^2}}{{{n^2}}} + \frac{{{a_{\max }}\left( {1 - p} \right)}m}{n^2}} \right)$, and ${q_{2,2}} = \frac{L}{n}\left( {1 - p} \right){a_{\max }} + \frac{\mu }{{n}}\left( {1 - p} \right){a_{\min }}$. In the above denotations, it is defined that ${a_i} = \sum\limits_{k = 1}^m {\frac{{{s_{i,k}}}}{{{d_k}}}}, \forall i$, ${a_{\max }} = \max \left\{ {{a_1},...,{a_n}} \right\}$, and ${a_{\min }} = \min \left\{ {{a_1},...,{a_n}} \right\}$.  
\end{lemma}

\begin{proof}
    Please see Appendix \ref{appendix xt-yitaIG-x*}.
\end{proof}

% \begin{lemma}
% \label{goco-2-goco-g}
% In GOCO, it holds that 
% \begin{align}
%       \label{equ-goco-2-goco-g}
%   \mathbb{E}&\left[ {\left. {\left\| {\eta {{\mathbf{G}}^t} \odot {{\mathbf{B}}^t}} \right\|_F^2} \right|{\Im _{t - 1}}} \right] \nonumber \\
%   \leq& {\eta ^2}\left( {1 - p} \right){G^2}\left[ {\sum\limits_{k = 1}^m {\frac{1}{{{d_k}}} + } \frac{1}{n}\left( {\sum\limits_{k = 1}^m {\frac{1}{{{d_k}}}} } \right)\left( {\sum\limits_{k = 1}^m {{d_k}} } \right) - \frac{m}{n}} \right],
%   \end{align}
%   \begin{align}
%       \label{equ-goco-2-goco-g1}
%  \mathbb{E}&\left[ {\left. {\left\| {\eta {{\mathbf{G}}^t} \odot {{\mathbf{B}}^t} - \eta \left( {{{\mathbf{G}}^t} \odot {{\mathbf{B}}^t}} \right)\frac{{{\mathbf{1}}{{\mathbf{1}}^T}}}{n}} \right\|_F^2} \right|{\Im _{t - 1}}} \right] \nonumber \\
%   \leq& {\eta ^2}\left( {1 - p} \right)\left( {1 - \frac{p}{n}} \right){G^2} \nonumber \\
%   & \times \left[ {\sum\limits_{k = 1}^m {\frac{1}{{{d_k}}} + } \frac{1}{n}\left( {\sum\limits_{k = 1}^m {\frac{1}{{{d_k}}}} } \right)\left( {\sum\limits_{k = 1}^m {{d_k}} } \right) - \frac{m}{n}} \right],
%   \end{align}
% where ${{\mathbf{B}}^t} = {\mathbf{1}}\left[ {I_1^t,...,I_n^t} \right]$, ${{\mathbf{G}}^t}= \left[ {{\mathbf{g}}_1^t,...,{\mathbf{g}}_n^t} \right]$, and $\odot$ is the Hadamard product operator. 
% \end{lemma}

% \begin{proof}
%     Please see Appendix \ref{appendix goco-2-goco-g} in \cite{chengxi2024preprint}.
% \end{proof}

\begin{lemma}
\label{GOCO-devia-optimal}
We can bound the deviation between the average parameter vector and the optimal parameter vector as 
    \begin{align}
    \label{overall1}
  \mathbb{E}&\left[ {\left. {{{\left\| {{{{\mathbf{\bar x}}}^{t + 1}} - {{\mathbf{x}}^*}} \right\|}^2}} \right|{\Im _{t - 1}}} \right] \nonumber \\
   \leq & \left( {1 + \Delta } \right){q_1}{\left\| {{{{\mathbf{\bar x}}}^t} - {{\mathbf{x}}^*}} \right\|^2} + \left( {1 + \Delta } \right){q_2}\left\| {{{{\mathbf{\bar X}}}^t} - {{\mathbf{X}}^t}} \right\|_F^2 \nonumber \\
   &+ \left( {1 + \Delta } \right)\frac{{4{\eta ^2}}}{{{n^2}}}\left( {p - {p^2}} \right){C^2}\sum\limits_{k = 1}^m {\left( {\frac{1}{{{d_k}}} - \frac{1}{n}} \right)}  \nonumber\\
  & + \left( {1 + \Delta } \right)2\eta \frac{m}{n}\left( {1 - p} \right)\left[ {f\left( {{{\mathbf{x}}^*}} \right) - f\left( {{{{\mathbf{\bar x}}}^t}} \right)} \right] \nonumber \\
   &+ \left( {1 + \Delta } \right)\frac{{{\eta ^2}}}{n}\left( {1 - p} \right)\sum\limits_{k = 1}^m {\frac{{\sigma _k^2}}{{{d_k}}}}  \nonumber \\
  &+\left( {1 + \frac{1}{\Delta }} \right)\frac{{{\gamma ^2}}}{n}\left( {1 - p} \right){\beta ^2}\left\| {{{\mathbf{X}}^t} - {{{\mathbf{\bar X}}}^t}} \right\|_F^2,\Delta  > 0.
\end{align} 
\end{lemma}
\begin{proof}
    Please see Appendix \ref{appendix GOCO-devia-optimal}.
\end{proof}

\begin{lemma}
\label{GOCO-divia1}
We can derive that 
\begin{align}
    \label{eq-goco-t+1}
 &\mathbb{E}\left[ {\left. {\left\| {{{\mathbf{X}}^{t + 1}} - {{{\mathbf{\bar X}}}^{t + 1}}} \right\|_F^2} \right|{\Im _{t - 1}}} \right]{\text{ }} \nonumber\\
 &\leq {w_1}\mathbb{E}\left[ {\left. {\left\| {{{{\mathbf{\bar X}}}^t} - {{\mathbf{X}}^t}} \right\|_F^2} \right|{\Im _{t - 1}}} \right] + {\kappa _0}{\eta ^2},
\end{align}
where ${\kappa _0} = \left( {1 - p} \right)\left( {1 + \frac{2}{{\gamma \rho }}} \right){G^2}{w_2}$, $0< w_1={\left( {1 - p} \right){{\left( {1 - \frac{{\gamma \rho }}{2}} \right)^2}} + p} <1$, and ${w_2} = \sum\limits_{k = 1}^m {\frac{1}{{{d_k}}} + } \frac{1}{n}\left( {\sum\limits_{k = 1}^m {\frac{1}{{{d_k}}}} } \right)\left( {\sum\limits_{k = 1}^m {{d_k}} } \right) - \frac{m}{n}$. 
\end{lemma}

\begin{proof}
    Please see Appendix \ref{appendix GOCO-divia1}.
\end{proof}

According to the above lemmas, we state the convergence performance of GOCO in the following theorem.  
\begin{theorem}
\label{theorem convergence_GOCO}
Under Assumptions \ref{assumption-loss-function}-\ref{subsets variation} and ${\left\| {{{\mathbf{X}}^0} - {{{\mathbf{\bar X}}}^0}} \right\|_F^2}=0$, with constant learning rate \(\eta  = \frac{{{\lambda _0}}}{{\sqrt T }}, \lambda _0>0\), for \(\mu  \gg \frac{{{\gamma ^2}{\kappa _0}{\beta ^2}}}{{m{a_{\min }}\left( {1 - p} \right)\left( {1 - {w_1}} \right)}}\), GOCO converges as
\begin{align}
     \label{goco-th10}
   \frac{1}{T}\sum\limits_{t = 0}^{T - 1} {\mathbb{E}\left[ {f\left( {{{{\mathbf{\bar x}}}^t}} \right) - f\left( {{{\mathbf{x}}^*}} \right)} \right]}  \approx  {\frac{{{\phi _1^1}}}{{\sqrt T }} + \frac{{{\phi _2^1}}}{T} + \frac{{{\phi _3^1}}}{{T\sqrt T }}} ,
\end{align}
  where
  \begin{align}
      \label{goco-th11}
  {\phi _1^1} =& \frac{n}{{2m{\lambda _0}\left( {1 - p} \right)}}{\left\| {{{{\mathbf{\bar x}}}^0} - {{\mathbf{x}}^*}} \right\|^2} \nonumber \\
   &+ \frac{{2{\lambda _0}}}{{mn}}p{C^2}\sum\limits_{k = 1}^m {\left( {\frac{1}{{{d_k}}} - \frac{1}{n}} \right)}  + \frac{{ {\lambda _0}}}{{2m}}\sum\limits_{k = 1}^m {\frac{{\sigma _k^2}}{{{d_k}}}}, \\
      \label{goco-th12}
  {\phi _2^1} =& \frac{{{q_{2,2}}\lambda _0^2}}{{1 - p}}\frac{{{\kappa _0}}}{{1 - {w_1}}}\frac{n}{{2m}},\\
      \label{goco-th13}
{\phi _3^1} =& \frac{{{q_{2,1}}\lambda _0^3}}{{1 - p}}\frac{{{\kappa _0}}}{{1 - {w_1}}}\frac{n}{{2m}}.
      \end{align}
\end{theorem}

\begin{proof}
    According to Lemma \ref{GOCO-devia-optimal} and Lemma \ref{GOCO-divia1}, we can take expectation on both sides of (\ref{overall1}) to obtain
\begin{align}
    \label{gocothth}
  \mathbb{E}&\left[ {{{\left\| {{{{\mathbf{\bar x}}}^{t + 1}} - {{\mathbf{x}}^*}} \right\|}^2}} \right] \nonumber \\
   \leq& \left( {1 + \Delta } \right){q_1}\mathbb{E}\left[ {{{\left\| {{{{\mathbf{\bar x}}}^t} - {{\mathbf{x}}^*}} \right\|}^2}} \right] \nonumber \\
  & + \left( {1 + \Delta } \right)\frac{{4{\eta ^2}}}{{{n^2}}}\left( {p - {p^2}} \right){C^2}\sum\limits_{k = 1}^m {\left( {\frac{1}{{{d_k}}} - \frac{1}{n}} \right)}  \nonumber \\
  & + \left( {1 + \Delta } \right)2\eta \frac{m}{n}\left( {1 - p} \right)\mathbb{E}\left[ {f\left( {{{\mathbf{x}}^*}} \right) - f\left( {{{{\mathbf{\bar x}}}^t}} \right)} \right] \nonumber \\
  & + \left( {1 + \Delta } \right)\frac{{{\eta ^2}}}{n}\left( {1 - p} \right)\sum\limits_{k = 1}^m {\frac{{\sigma _k^2}}{{{d_k}}}}  \nonumber\\
  & + \left[ {\left( {1 + \frac{1}{\Delta }} \right)\frac{{{\gamma ^2}}}{n}\left( {1 - p} \right){\beta ^2} + \left( {1 + \Delta } \right){q_2}} \right]\frac{{{\kappa _0}{\eta ^2}}}{{{w_1} - 1}}w_1^t \nonumber \\
  & - \left[ {\left( {1 + \frac{1}{\Delta }} \right)\frac{{{\gamma ^2}}}{n}\left( {1 - p} \right){\beta ^2} + \left( {1 + \Delta } \right){q_2}} \right]\frac{{{\kappa _0}{\eta ^2}}}{{{w_1} - 1}}.
\end{align}
For $T > {\left[ {\frac{{{\lambda _0}}}{{\frac{{{q_{1,1}} - \sqrt {\max \left\{ {q_{1,1}^2 - 4{q_{1,2}},0} \right\}} }}{{2{q_{1,2}}}}}}} \right]^2}$, if we set $\eta  = \frac{{{\lambda _0}}}{{\sqrt T }}, \lambda _0>0$, it holds that $0 < 1 - {q_{1,1}}\eta {\text{ }} + {q_{1,2}}{\eta ^2} < 1$. Based on that, let us set 
  \begin{align}
      \label{goco-th3}
 \Delta  = \frac{1}{{{q_1}}} - 1 = \frac{{{q_{1,1}}\eta  - {q_{1,2}}{\eta ^2}}}{{1 - {q_{1,1}}\eta  + {q_{1,2}}{\eta ^2}}}>0.
  \end{align}  
  Combing (\ref{gocothth}) and (\ref{goco-th3}) and summing over $T$ iterations yield 
  \begin{align}
      \label{goco-th6}
  \frac{1}{T}&\sum\limits_{t = 0}^{T - 1} {\mathbb{E}\left[ {f\left( {{{{\mathbf{\bar x}}}^t}} \right) - f\left( {{{\mathbf{x}}^*}} \right)} \right]}  \nonumber \\
   \leq& \frac{n}{{2m\sqrt T {\lambda _0}\left( {1 - p} \right)}}{\left\| {{{{\mathbf{\bar x}}}^0} - {{\mathbf{x}}^*}} \right\|^2} \nonumber \\
  & + \frac{{2{\lambda _0}}}{{mn\sqrt T }}p{C^2}\sum\limits_{k = 1}^m {\left( {\frac{1}{{{d_k}}} - \frac{1}{n}} \right)}   + \frac{{ {\lambda _0}}}{{2m\sqrt T }}\sum\limits_{k = 1}^m {\frac{{\sigma _k^2}}{{{d_k}}}}  \nonumber \\
  & + \frac{{{\beta ^2}}}{{{q_{1,1}} - {q_{1,2}}\frac{{{\lambda _0}}}{{\sqrt T }}}}\frac{{{\gamma ^2}}}{2m}\frac{{{\kappa _0}}}{{1 - {w_1}}} \nonumber \\
  & + \frac{{{q_{2,1}}{{\left( {\frac{{{\lambda _0}}}{{\sqrt T }}} \right)}^3} + {q_{2,2}}{{\left( {\frac{{{\lambda _0}}}{{\sqrt T }}} \right)}^2}}}{{1 - p}}\frac{{{n\kappa _0}}}{{2m(1 - {w_1})}}.
  \end{align} 
For \(\mu  \gg \frac{{{\gamma ^2}{\kappa _0}{\beta ^2}}}{{m{a_{\min }}\left( {1 - p} \right)\left( {1 - {w_1}} \right)}}\), we can easily derive Theorem \ref{theorem convergence_GOCO} from (\ref{goco-th6}). 
\end{proof}
\begin{remark}
In Theorem \ref{theorem convergence_GOCO}, the condition of fairly large \(\mu\) is adopted, which has substantial practical relevance. For instance, with loss functions containing quadratic penalty terms, strong convexity of the loss functions is enhanced naturally. By selecting a high coefficient for the penalty terms, the above condition can be reasonably achieved. 
\end{remark}

\section{Simulations}
\label{simulations}
In this section, we test the performance of GOCO on a linear regression problem with loss function expressed as $f\left( {\mathbf{x}} \right)= \frac{1}{m}\sum\limits_{k = 1}^m {{f_k}\left( {\mathbf{x}} \right)}$, where ${f_k}\left( {\mathbf{x}} \right) = \frac{1}{2}{\left( {\left\langle {{\mathbf{x}},{{\mathbf{z}}_k}} \right\rangle  - {y_k}} \right)^2}$, \({{\mathbf{z}}_k} \in {\mathbb{R}^{100}}\), \(m = 16\), \(y_k \in {\mathbb{R}}\), $\forall k$ and \(\mathbf{x} \in {\mathbb{R}^{100}}\) is the parameter vector. Here, we generate the elements in \(\left\{ {{{\mathbf{z}}_1},...,{{\mathbf{z}}_n}} \right\}\) as independent random variables using normal distribution \(\mathcal{N}\left( {0,100} \right)\). With a random vector \({\mathbf{\overset{\lower0.5em\hbox{$\smash{\scriptscriptstyle\frown}$}}{x} }}\) containing 100 integers randomly selected between 1 and 10, we set $y_k\sim\mathcal{N}\left( {\left\langle {{{\mathbf{z}}_k},{\mathbf{\overset{\lower0.5em\hbox{$\smash{\scriptscriptstyle\frown}$}}{x} }}} \right\rangle ,1} \right)$, $\forall k$. The DEL system contains \(n = 16\) devices and they formulate a ring communication graph. To distribute the training data, as an approximation of the pair-wise balanced manner, the training subsets associated with \(\left\{ f_k(\mathbf{x}), \forall k \right\}\) are distributed uniformly at random by selecting a number of devices to hold each training subset. For our simulations, stochastic gradients are generated directly based on the expression of \( f_k(\mathbf{x}) \) and \(\nabla F_k(\mathbf{x}, \xi_k)\) follows a normal distribution \(\mathcal{N}\left( \nabla f_k(\mathbf{x}), \sigma_0^2\mathbf{I} \right)\), \(\forall \mathbf{x} \in \mathbb{R}^{100}\), \(k \in [m]\), where \(\mathbf{I}\) is the identity matrix, and \(\sigma_0 = 1\). We fix $p=0.2$, \({d_k} = 3,k \in [m]\), $\gamma=0.05$ and $\eta=0.0001$. We initialize the local parameter vectors to be all-zero vectors. 

To test the learning performance of GOCO, in Fig. \ref{fig: ring uncom}, the training loss of different methods is depicted as a function of the number of transmitted bits, where transmitting each element uses 64 bits. The two baseline methods are GOCO on all connected graph and Ignore Stragglers method (IS). The former assumes that all the devices can communicate with all the others, which is equivalent to \cite{bitar2020stochastic}. The latter distributes the training subsets without redundancy. From Fig. \ref{fig: ring uncom}, We can observe that GOCO achieves superior learning performance compared with the baseline methods, given the same level of communication overhead. This effectively demonstrates the value of our work.

\begin{figure}[b]

    \centering
        \includegraphics[width=0.8\linewidth]{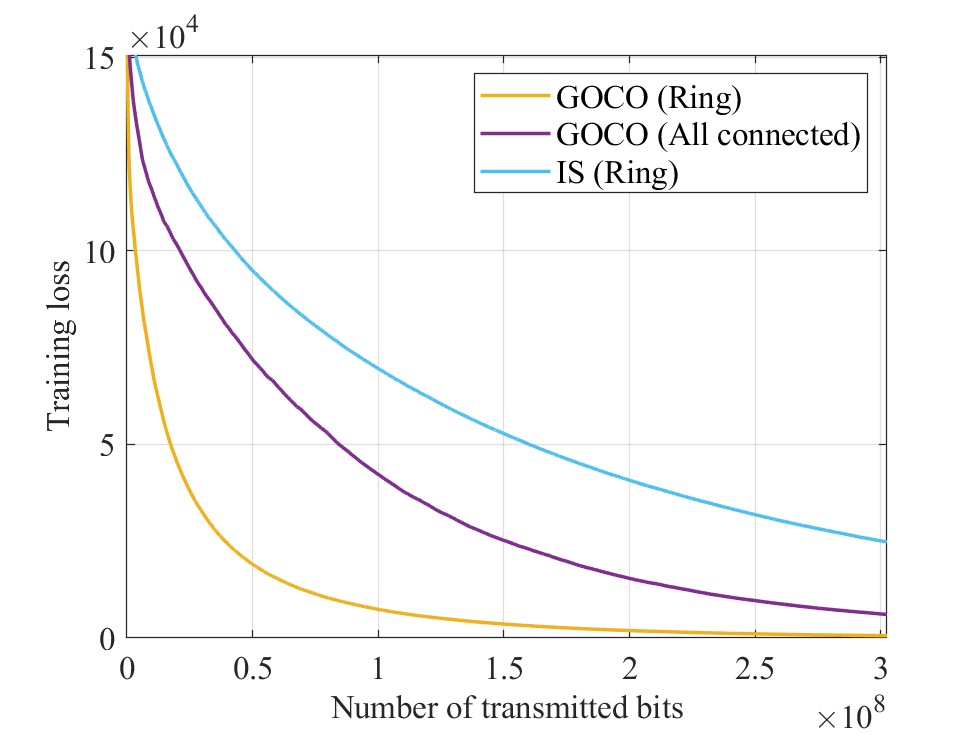}
        \caption{Training loss as a function of the number of transmitted bits of different methods.}
        \label{fig: ring uncom}
\end{figure}

\section{Conclusions}
\label{conclusions}
In this paper, the DEL problem in the presence of stragglers was addressed. To overcome the limitation of the existing GC techniques that they can not be applied under decentralized scenarios, we proposed GOCO by combining the advantages of SGC and gossip-based approaches to evade stragglers. The convergence performance of GOCO was analytically described for strongly convex loss functions and the superiority of GOCO was demonstrated via simulation results. In the future, we plan to further reduce the communication overhead of the proposed method by incorporating communication compression techniques.

%%%%%%
%% To balance the columns at the last page of the paper use this
%% command:
%%
%\enlargethispage{-1.2cm} 
%%
%% If the balancing should occur in the middle of the references, use
%% the following trigger:
%%
%%
%% which triggers a \newpage (i.e., new column) just before the given
%% reference number. Note that you need to adapt this if you modify
%% the paper.  The "triggered" command can be changed if desired:
%%
%\IEEEtriggercmd{\enlargethispage{-20cm}}
%%
%%%%%%

%%%%%%
%% References:
%% We recommend the usage of BibTeX:
%%
%\bibliographystyle{IEEEtran}
%\bibliography{definitions,bibliofile}
%%
%% where we here have assumed the existence of the files
%% definitions.bib and bibliofile.bib.
%% BibTeX documentation can be obtained at:
%% http://www.ctan.org/tex-archive/biblio/bibtex/contrib/doc/
%%%%%%

\bibliographystyle{IEEEtran}   
\bibliography{reference} 
\clearpage
\appendices

\section{Proof of Lemma \ref{xt-yitaIG-x*}}
  \label{appendix xt-yitaIG-x*}

For the proposed methods, based on (\ref{g}), we can derive (\ref{a1_1}) by using simple expansions. 
\begin{figure*} % The 'b' option will place the figure at the bottom of the page
\begin{align}
    \label{a1_1}
  \mathbb{E}&\left[ {\left. {{{\left\| {{{{\mathbf{\bar x}}}^t} - \frac{\eta }{n}\sum\limits_{i = 1}^n {I_i^t} {\mathbf{g}}_i^t - {{\mathbf{x}}^*}} \right\|}^2}} \right|{\Im _{t - 1}}} \right] \nonumber \\
   =& \mathbb{E}\left[ {\left. {{{\left\| {{{{\mathbf{\bar x}}}^t} - \frac{\eta }{n}\sum\limits_{i = 1}^n {I_i^t} \sum\limits_{k = 1}^m {\frac{1}{{{d_k}}}\nabla {F_k}({\mathbf{x}}_i^t,\xi _{i,k}^t){s_{i,k}}}  - {{\mathbf{x}}^*}} \right\|}^2}} \right|{\Im _{t - 1}}} \right] \nonumber \\
   = &\mathbb{E}\left[ {\left. {{{\left\| {{{{\mathbf{\bar x}}}^t} - {{\mathbf{x}}^*} - \frac{\eta }{n}\sum\limits_{i = 1}^n {I_i^t} \sum\limits_{k = 1}^m {\frac{1}{{{d_k}}}\nabla {f_k}({\mathbf{x}}_i^t){s_{i,k}}}  + \frac{\eta }{n}\sum\limits_{i = 1}^n {I_i^t} \sum\limits_{k = 1}^m {\frac{1}{{{d_k}}}\nabla {f_k}({\mathbf{x}}_i^t){s_{i,k}}}  - \frac{\eta }{n}\sum\limits_{i = 1}^n {I_i^t} \sum\limits_{k = 1}^m {\frac{1}{{{d_k}}}\nabla {F_k}({\mathbf{x}}_i^t,\xi _{i,k}^t){s_{i,k}}} } \right\|}^2}} \right|{\Im _{t - 1}}} \right] \nonumber \\
   = &\mathbb{E}\left[ {\left. {{{\left\| {{{{\mathbf{\bar x}}}^t} - {{\mathbf{x}}^*} - \frac{\eta }{n}\sum\limits_{i = 1}^n {I_i^t} \sum\limits_{k = 1}^m {\frac{1}{{{d_k}}}\nabla {f_k}({\mathbf{x}}_i^t){s_{i,k}}} } \right\|}^2}} \right|{\Im _{t - 1}}} \right] \nonumber \\
  & + \frac{{{\eta ^2}}}{{{n^2}}}\mathbb{E}\left[ {\left. {{{\left\| {\sum\limits_{i = 1}^n {I_i^t} \sum\limits_{k = 1}^m {\frac{1}{{{d_k}}}\nabla {f_k}({\mathbf{x}}_i^t){s_{i,k}}}  - \sum\limits_{i = 1}^n {I_i^t} \sum\limits_{k = 1}^m {\frac{1}{{{d_k}}}\nabla {F_k}({\mathbf{x}}_i^t,\xi _{i,k}^t){s_{i,k}}} } \right\|}^2}} \right|{\Im _{t - 1}}} \right] \nonumber \\
 &+ \frac{{2\eta }}{n}\mathbb{E}\left[ {\left. {\left\langle {{{{\mathbf{\bar x}}}^t} - {{\mathbf{x}}^*} - \frac{\eta }{n}\sum\limits_{i = 1}^n {I_i^t} \sum\limits_{k = 1}^m {\frac{1}{{{d_k}}}\nabla {f_k}({\mathbf{x}}_i^t){s_{i,k}}} ,\sum\limits_{i = 1}^n {I_i^t} \sum\limits_{k = 1}^m {\frac{1}{{{d_k}}}\nabla {f_k}({\mathbf{x}}_i^t){s_{i,k}}}  - \sum\limits_{i = 1}^n {I_i^t} \sum\limits_{k = 1}^m {\frac{1}{{{d_k}}}\nabla {F_k}({\mathbf{x}}_i^t,\xi _{i,k}^t){s_{i,k}}} } \right\rangle } \right|{\Im _{t - 1}}} \right]. 
\end{align}
\end{figure*}
According to (\ref{unbias}), the last term in (\ref{a1_1}) equals zero by noting that 
\begin{align}
\label{t3}
  &\mathbb{E}\left[ {\left. {\sum\limits_{k = 1}^m {\frac{1}{{{d_k}}}\nabla {f_k}({\mathbf{x}}_i^t){s_{i,k}}}  - \sum\limits_{k = 1}^m {\frac{1}{{{d_k}}}\nabla {F_k}({\mathbf{x}}_i^t,\xi _{i,k}^t){s_{i,k}}} } \right|{\Im _{t - 1}}} \right] \nonumber \\
  & = \sum\limits_{k = 1}^m {\frac{{{s_{i,k}}}}{{{d_k}}}\mathbb{E}\left[ {\left. {\nabla {f_k}({\mathbf{x}}_i^t) - \nabla {F_k}({\mathbf{x}}_i^t,\xi _{i,k}^t)} \right|{\Im _{t - 1}}} \right]}  \nonumber \\
  & = 0, 
\end{align}
and that sampling training data for the computation of stochastic gradients and the straggler behaviour of the devices are independent. According to (\ref{Iit_bernoulli}) and (\ref{boundvar1}), the second term in (\ref{a1_1}) can be rewritten as (\ref{a1_1_2}), where the following basic inequality is employed:
 \begin{align}
    \label{ineq2}
    \left\| \sum_{i=1}^{n} \mathbf{a}_i \right\|^2 \leq n \sum_{i=1}^{n} \left\| \mathbf{a}_i \right\|^2, &\forall\mathbf{a}_i \in \mathbb{R}^d.
    \end{align}
and \(\sum\limits_{i = 1}^n {{s_{i,k}}}  = {d_k}\) is noted. 
    
\begin{figure*}
\begin{align}
    \label{a1_1_2}
&  \frac{{{\eta ^2}}}{{{n^2}}}\mathbb{E}\left[ {\left. {{{\left\| {\sum\limits_{i = 1}^n {I_i^t} \sum\limits_{k = 1}^m {\frac{1}{{{d_k}}}\nabla {f_k}({\mathbf{x}}_i^t){s_{i,k}}}  - \sum\limits_{i = 1}^n {I_i^t} \sum\limits_{k = 1}^m {\frac{1}{{{d_k}}}\nabla {F_k}({\mathbf{x}}_i^t,\xi _{i,k}^t){s_{i,k}}} } \right\|}^2}} \right|{\Im _{t - 1}}} \right] \nonumber \\
 &  = \frac{{{\eta ^2}}}{{{n^2}}}\mathbb{E}\left[ {\left. {{{\left\| {\sum\limits_{i = 1}^n {I_i^t} \sum\limits_{k = 1}^m {\frac{{{s_{i,k}}}}{{{d_k}}}\left( {\nabla {f_k}\left( {{\mathbf{x}}_i^t} \right) - \nabla {F_k}\left( {{\mathbf{x}}_i^t,\xi _{i,k}^t} \right)} \right)} } \right\|}^2}} \right|{\Im _{t - 1}}} \right] \nonumber \\
 &  \leq \frac{{{\eta ^2}}}{{{n}}}\sum\limits_{i = 1}^n {\mathbb{E}\left[ {\left. {I_i^t{{\left\| {\sum\limits_{k = 1}^m {\frac{{{s_{i,k}}}}{{{d_k}}}\left( {\nabla {f_k}\left( {{\mathbf{x}}_i^t} \right) - \nabla {F_k}\left( {{\mathbf{x}}_i^t,\xi _{i,k}^t} \right)} \right)} } \right\|}^2}} \right|{\Im _{t - 1}}} \right]}  \nonumber \\
 &  = \frac{{{\eta ^2}}}{{{n}}}\sum\limits_{i = 1}^n {\mathbb{E}\left[ {\left. {{{\left\| {\sum\limits_{k = 1}^m {\frac{{{s_{i,k}}}}{{{d_k}}}\left( {\nabla {f_k}\left( {{\mathbf{x}}_i^t} \right) - \nabla {F_k}\left( {{\mathbf{x}}_i^t,\xi _{i,k}^t} \right)} \right)} } \right\|}^2}} \right|{\Im _{t - 1}}} \right]\mathbb{E}\left( {\left. {I_i^t} \right|{\Im _{t - 1}}} \right)}  \nonumber \\
 &  = \frac{{{\eta ^2}}}{{{n}}}\left( {1 - p} \right)\sum\limits_{i = 1}^n {\mathbb{E}\left[ {\left. {{{\left\| {\sum\limits_{k = 1}^m {\frac{{{s_{i,k}}}}{{{d_k}}}\left( {\nabla {f_k}\left( {{\mathbf{x}}_i^t} \right) - \nabla {F_k}\left( {{\mathbf{x}}_i^t,\xi _{i,k}^t} \right)} \right)} } \right\|}^2}} \right|{\Im _{t - 1}}} \right]} 
 \leq \frac{{{\eta ^2}}}{{{n}}}\left( {1 - p} \right)\sum\limits_{i = 1}^n {\sum\limits_{k = 1}^m {\frac{{{s_{i,k}}}}{{d_k^2}}\sigma _k^2} } 
 = \frac{{{\eta ^2}}}{{{n}}}\left( {1 - p} \right)\sum\limits_{k = 1}^m {\frac{{\sigma _k^2}}{{{d_k}}}}.
\end{align}
\end{figure*}
The first term in (\ref{a1_1}) can be rewritten as
\begin{align}
    \label{a1_1_de}
  \mathbb{E}&\left[ {\left. {{{\left\| {{{{\mathbf{\bar x}}}^t} - {{\mathbf{x}}^*} - \frac{\eta }{n}\sum\limits_{i = 1}^n {I_i^t} \sum\limits_{k = 1}^m {\frac{1}{{{d_k}}}\nabla {f_k}({\mathbf{x}}_i^t){s_{i,k}}} } \right\|}^2}} \right|{\Im _{t - 1}}} \right] \nonumber \\
   =&\mathbb{E}\left( {\left. {{{\left\| {{{{\mathbf{\bar x}}}^t} - {{\mathbf{x}}^*}} \right\|}^2}} \right|{\Im _{t - 1}}} \right) \nonumber \\
  &+ \mathbb{E}\left[ {\left. {{{\left\| {\frac{\eta }{n}\sum\limits_{i = 1}^n {I_i^t} \sum\limits_{k = 1}^m {\frac{1}{{{d_k}}}\nabla {f_k}({\mathbf{x}}_i^t){s_{i,k}}} } \right\|}^2}} \right|{\Im _{t - 1}}} \right] \nonumber \\
& - \frac{{2\eta }}{n}\mathbb{E}\left[ {\left. {\left\langle {{{{\mathbf{\bar x}}}^t} - {{\mathbf{x}}^*},\sum\limits_{i = 1}^n {I_i^t} \sum\limits_{k = 1}^m {\frac{1}{{{d_k}}}\nabla {f_k}({\mathbf{x}}_i^t){s_{i,k}}} } \right\rangle } \right|{\Im _{t - 1}}} \right].
\end{align}
In (\ref{a1_1_de}), for the second term, we have
\begin{align}
    \label{a11de2}
  \mathbb{E}&\left[ {\left. {{{\left\| {\frac{\eta }{n}\sum\limits_{i = 1}^n {I_i^t} \sum\limits_{k = 1}^m {\frac{1}{{{d_k}}}\nabla {f_k}({\mathbf{x}}_i^t){s_{i,k}}} } \right\|}^2}} \right|{\Im _{t - 1}}} \right] \nonumber \\
   = &\frac{{{\eta ^2}}}{{{n^2}}}\mathbb{E}\left[ {\left\| {\sum\limits_{i = 1}^n {I_i^t} \sum\limits_{k = 1}^m {\frac{{{s_{i,k}}}}{{{d_k}}}\left[ {\nabla {f_k}({\mathbf{x}}_i^t) - \nabla {f_k}({{{\mathbf{\bar x}}}^t})} \right.} } \right.} \right. \nonumber \\
  &\left. {\left. {{{\left. {\left. { + \nabla {f_k}({{{\mathbf{\bar x}}}^t}) - \nabla f({{\mathbf{x}}^*})} \right]} \right\|}^2}} \right|{\Im _{t - 1}}} \right] \nonumber \\
   \leq &\frac{{2{\eta ^2}}}{{{n^2}}}\mathbb{E}\left[ {\left. {{{\left\| {\sum\limits_{i = 1}^n {I_i^t} \sum\limits_{k = 1}^m {\frac{{{s_{i,k}}}}{{{d_k}}}\left[ {\nabla {f_k}({\mathbf{x}}_i^t) - \nabla {f_k}({{{\mathbf{\bar x}}}^t})} \right]} } \right\|}^2}} \right|{\Im _{t - 1}}} \right] \nonumber \\
  & + \frac{{2{\eta ^2}}}{{{n^2}}}\mathbb{E}\left[ {\left. {{{\left\| {\sum\limits_{i = 1}^n {I_i^t} \sum\limits_{k = 1}^m {\frac{{{s_{i,k}}}}{{{d_k}}}\left[ {\nabla {f_k}({{{\mathbf{\bar x}}}^t}) - \nabla f({{\mathbf{x}}^*})} \right]} } \right\|}^2}} \right|{\Im _{t - 1}}} \right], 
\end{align}
where the equality holds since $\nabla f({{\mathbf{x}}^*}) = 0$ and the inequality is derived from (\ref{ineq2}). For the second term in (\ref{a11de2}), we can provide the bound given as (\ref{a1_1_3}),
\begin{figure*}
\begin{align}
    \label{a1_1_3}
  &\frac{{2{\eta ^2}}}{{{n^2}}}\mathbb{E}\left[ {\left. {{{\left\| {\sum\limits_{i = 1}^n {I_i^t} \sum\limits_{k = 1}^m {\frac{{{s_{i,k}}}}{{{d_k}}}\left[ {\nabla {f_k}({{{\mathbf{\bar x}}}^t}) - \nabla f({{\mathbf{x}}^*})} \right]} } \right\|}^2}} \right|{\Im _{t - 1}}} \right] \nonumber \\
   = &\frac{{2{\eta ^2}}}{{{n^2}}}\mathbb{E}\left[ {\left. {\sum\limits_{i = 1}^n {I_i^t} \sum\limits_{{k_1} = 1}^m {\sum\limits_{{k_2} = 1}^m {\left\langle {\frac{{{s_{i,{k_1}}}}}{{{d_{{k_1}}}}}\left[ {\nabla {f_{{k_1}}}({{{\mathbf{\bar x}}}^t}) - \nabla f({{\mathbf{x}}^*})} \right],\frac{{{s_{i,{k_2}}}}}{{{d_{{k_2}}}}}\left[ {\nabla {f_{{k_2}}}({{{\mathbf{\bar x}}}^t}) - \nabla f({{\mathbf{x}}^*})} \right]} \right\rangle } } } \right|{\Im _{t - 1}}} \right] \nonumber \\
  & + \frac{{2{\eta ^2}}}{{{n^2}}}\mathbb{E}\left[ {\left. {\sum\limits_{{i_1} = 1}^n {\sum\limits_{{i_2} \ne {i_1}}^n {I_{{i_1}}^tI_{{i_2}}^t} } \sum\limits_{{k_1} = 1}^m {\sum\limits_{{k_2} = 1}^m {\left\langle {\frac{{{s_{{i_1},{k_1}}}}}{{{d_{{k_1}}}}}\left[ {\nabla {f_{{k_1}}}({{{\mathbf{\bar x}}}^t}) - \nabla f({{\mathbf{x}}^*})} \right],\frac{{{s_{{i_2},{k_2}}}}}{{{d_{{k_2}}}}}\left[ {\nabla {f_{{k_2}}}({{{\mathbf{\bar x}}}^t}) - \nabla f({{\mathbf{x}}^*})} \right]} \right\rangle } } } \right|{\Im _{t - 1}}} \right] \nonumber \\
   = &\frac{{2{\eta ^2}}}{{{n^2}}}\left( {1 - p} \right)\sum\limits_{i = 1}^n {\sum\limits_{{k_1} = 1}^m {\sum\limits_{{k_2} = 1}^m {\left\langle {\frac{{{s_{i,{k_1}}}}}{{{d_{{k_1}}}}}\left[ {\nabla {f_{{k_1}}}({{{\mathbf{\bar x}}}^t}) - \nabla f({{\mathbf{x}}^*})} \right],\frac{{{s_{i,{k_2}}}}}{{{d_{{k_2}}}}}\left[ {\nabla {f_{{k_2}}}({{{\mathbf{\bar x}}}^t}) - \nabla f({{\mathbf{x}}^*})} \right]} \right\rangle } } }  \nonumber \\
  & + \frac{{2{\eta ^2}}}{{{n^2}}}{\left( {1 - p} \right)^2}\sum\limits_{{i_1} = 1}^n {\sum\limits_{{i_2} \ne {i_1}}^n {\sum\limits_{{k_1} = 1}^m {\sum\limits_{{k_2} = 1}^m {\left\langle {\frac{{{s_{{i_1},{k_1}}}}}{{{d_{{k_1}}}}}\left[ {\nabla {f_{{k_1}}}({{{\mathbf{\bar x}}}^t}) - \nabla f({{\mathbf{x}}^*})} \right],\frac{{{s_{{i_2},{k_2}}}}}{{{d_{{k_2}}}}}\left[ {\nabla {f_{{k_2}}}({{{\mathbf{\bar x}}}^t}) - \nabla f({{\mathbf{x}}^*})} \right]} \right\rangle } } } }  \nonumber \\
   = &\frac{{2{\eta ^2}}}{{{n^2}}}\left( {p - {p^2}} \right)\sum\limits_{i = 1}^n {\sum\limits_{k = 1}^m {\frac{{{s_{i,k}}}}{{d_k^2}}\left\langle {\nabla {f_k}({{{\mathbf{\bar x}}}^t}) - \nabla f({{\mathbf{x}}^*}),\nabla {f_k}({{{\mathbf{\bar x}}}^t}) - \nabla f({{\mathbf{x}}^*})} \right\rangle } }  \nonumber \\
 &+ \frac{{2{\eta ^2}}}{{{n^2}}}\left( {p - {p^2}} \right)\sum\limits_{i = 1}^n {\sum\limits_{{k_1} = 1}^m {\sum\limits_{{k_2} \ne {k_1}}^m {\frac{{{s_{i,{k_1}}}{s_{i,{k_2}}}}}{{{d_{{k_1}}}{d_{{k_2}}}}}\left\langle {\nabla {f_{{k_1}}}({{{\mathbf{\bar x}}}^t}) - \nabla f({{\mathbf{x}}^*}),\nabla {f_{{k_2}}}({{{\mathbf{\bar x}}}^t}) - \nabla f({{\mathbf{x}}^*})} \right\rangle } } }  \nonumber \\
 &+ \frac{{2{\eta ^2}}}{{{n^2}}}{\left( {1 - p} \right)^2}{\left\| {\sum\limits_{i = 1}^n {\sum\limits_{k = 1}^m {\frac{{{s_{i,k}}}}{{{d_k}}}\left[ {\nabla {f_k}({{{\mathbf{\bar x}}}^t}) - \nabla f({{\mathbf{x}}^*})} \right]} } } \right\|^2} \nonumber \\
   =& \frac{{2{\eta ^2}}}{{{n^2}}}\left( {p - {p^2}} \right)\sum\limits_{k = 1}^m {\left( {\frac{1}{{{d_k}}} - \frac{1}{n}} \right){{\left\| {\nabla {f_k}({{{\mathbf{\bar x}}}^t}) - \nabla f({{\mathbf{x}}^*})} \right\|}^2}}  \nonumber \\
 & + \left[ {\frac{{2{\eta ^2}}}{{{n^3}}}\left( {p - {p^2}} \right) + \frac{{2{\eta ^2}}}{{{n^2}}}{{\left( {1 - p} \right)}^2}} \right]{\left\| {\sum\limits_{k = 1}^m {\left[ {\nabla {f_k}({{{\mathbf{\bar x}}}^t}) - \nabla {f_k}({{\mathbf{x}}^*})} \right]} } \right\|^2} \nonumber \\
   = &\frac{{2{\eta ^2}}}{{{n^2}}}\left( {p - {p^2}} \right)\sum\limits_{k = 1}^m {\left( {\frac{1}{{{d_k}}} - \frac{1}{n}} \right){{\left\| {\nabla {f_k}({{{\mathbf{\bar x}}}^t}) - \nabla {f_k}({{\mathbf{x}}^*}) + \nabla {f_k}({{\mathbf{x}}^*}) - \nabla f({{\mathbf{x}}^*})} \right\|}^2}}  \nonumber \\
  & + \left[ {\frac{{2{\eta ^2}}}{{{n^3}}}\left( {p - {p^2}} \right) + \frac{{2{\eta ^2}}}{{{n^2}}}{{\left( {1 - p} \right)}^2}} \right]{\left\| {\sum\limits_{k = 1}^m {\left[ {\nabla {f_k}({{{\mathbf{\bar x}}}^t}) - \nabla {f_k}({{\mathbf{x}}^*})} \right]} } \right\|^2} \nonumber \\
   \leq &\frac{{4{\eta ^2}}}{{{n^2}}}\left( {p - {p^2}} \right)\sum\limits_{k = 1}^m {\left( {\frac{1}{{{d_k}}} - \frac{1}{n}} \right)\left( {{L^2}{{\left\| {{{{\mathbf{\bar x}}}^t} - {{\mathbf{x}}^*}} \right\|}^2} + {C^2}} \right)}  + \left[ {\frac{{2{\eta ^2}}}{{{n^3}}}\left( {p - {p^2}} \right) + \frac{{2{\eta ^2}}}{{{n^2}}}{{\left( {1 - p} \right)}^2}} \right]m\sum\limits_{k = 1}^m {{L^2}{{\left\| {{{{\mathbf{\bar x}}}^t} - {{\mathbf{x}}^*}} \right\|}^2}} \nonumber \\
   =& \left[ {\frac{{2{\eta ^2}}}{{{n^3}}}\left( {p - {p^2}} \right){m^2}{L^2} + \frac{{2{\eta ^2}}}{{{n^2}}}{{\left( {1 - p} \right)}^2}{m^2}{L^2} + \frac{{4{\eta ^2}}}{{{n^2}}}\left( {p - {p^2}} \right){L^2}\sum\limits_{k = 1}^m {\left( {\frac{1}{{{d_k}}} - \frac{1}{n}} \right)} } \right]{\left\| {{{{\mathbf{\bar x}}}^t} - {{\mathbf{x}}^*}} \right\|^2} \nonumber\\
   &+ \frac{{4{\eta ^2}}}{{{n^2}}}\left( {p - {p^2}} \right){C^2}\sum\limits_{k = 1}^m {\left( {\frac{1}{{{d_k}}} - \frac{1}{n}} \right)} .
\end{align}
\end{figure*}
where \(\sum\limits_{i = 1}^n {{s_{i,k}}}  = {d_k}\), \(\sum\limits_{i = 1}^n {{s_{i,{k_1}}}{s_{i,{k_2}}}}  = \frac{{{d_{{k_1}}}{d_{{k_2}}}}}{n}\) ($k_1\neq k_2$), Assumptions \ref{assumption-loss-function}-\ref{subsets variation} and the basic inequality in (\ref{ineq2}) are employed. For the first term in (\ref{a11de2}),
\begin{figure*}
\begin{align}
    \label{a11de21}
 & \frac{{2{\eta ^2}}}{{{n^2}}}\mathbb{E}\left[ {\left. {{{\left\| {\sum\limits_{i = 1}^n {I_i^t} \sum\limits_{k = 1}^m {\frac{{{s_{i,k}}}}{{{d_k}}}\left[ {\nabla {f_k}({\mathbf{x}}_i^t) - \nabla {f_k}({{{\mathbf{\bar x}}}^t})} \right]} } \right\|}^2}} \right|{\Im _{t - 1}}} \right] \nonumber \\
   = &\frac{{2{\eta ^2}}}{{{n^2}}}\mathbb{E}\left[ {\left. {\sum\limits_{i = 1}^n {I_i^t\sum\limits_{{k_1} = 1}^m {\sum\limits_{{k_2} = 1}^m {\frac{{{s_{i,{k_1}}}}}{{{d_{{k_1}}}}}\frac{{{s_{i,{k_2}}}}}{{{d_{{k_2}}}}}\left\langle {\left[ {\nabla {f_{{k_1}}}({\mathbf{x}}_i^t) - \nabla {f_{{k_1}}}({{{\mathbf{\bar x}}}^t})} \right],\left[ {\nabla {f_{{k_2}}}({\mathbf{x}}_i^t) - \nabla {f_{{k_2}}}({{{\mathbf{\bar x}}}^t})} \right]} \right\rangle } } } } \right|{\Im _{t - 1}}} \right] \nonumber \\
  & + \frac{{2{\eta ^2}}}{{{n^2}}}\mathbb{E}\left[ {\left. {\sum\limits_{{i_1} = 1}^n {\sum\limits_{{i_2} \ne {i_1}}^n {I_{{i_1}}^tI_{{i_2}}^t} \sum\limits_{{k_1} = 1}^m {\sum\limits_{{k_2} = 1}^m {\frac{{{s_{{i_1},{k_1}}}}}{{{d_{{k_1}}}}}\frac{{{s_{{i_2},{k_2}}}}}{{{d_{{k_2}}}}}\left\langle {\left[ {\nabla {f_{{k_1}}}({\mathbf{x}}_{{i_1}}^t) - \nabla {f_{{k_1}}}({{{\mathbf{\bar x}}}^t})} \right],\left[ {\nabla {f_{{k_2}}}({\mathbf{x}}_{{i_2}}^t) - \nabla {f_{{k_2}}}({{{\mathbf{\bar x}}}^t})} \right]} \right\rangle } } } } \right|{\Im _{t - 1}}} \right] \nonumber \\
   =& \frac{{2{\eta ^2}}}{{{n^2}}}\left( {p - {p^2}} \right)\sum\limits_{i = 1}^n {\sum\limits_{{k_1} = 1}^m {\sum\limits_{{k_2} = 1}^m {\frac{{{s_{i,{k_1}}}}}{{{d_{{k_1}}}}}\frac{{{s_{i,{k_2}}}}}{{{d_{{k_2}}}}}\left\langle {\left[ {\nabla {f_{{k_1}}}({\mathbf{x}}_i^t) - \nabla {f_{{k_1}}}({{{\mathbf{\bar x}}}^t})} \right],\left[ {\nabla {f_{{k_2}}}({\mathbf{x}}_i^t) - \nabla {f_{{k_2}}}({{{\mathbf{\bar x}}}^t})} \right]} \right\rangle } } }  \nonumber \\
 & + \frac{{2{\eta ^2}}}{{{n^2}}}{\left( {1 - p} \right)^2}\sum\limits_{{i_1} = 1}^n {\sum\limits_{{i_2} = 1}^n {\sum\limits_{{k_1} = 1}^m {\sum\limits_{{k_2} = 1}^m {\frac{{{s_{{i_1},{k_1}}}}}{{{d_{{k_1}}}}}\frac{{{s_{{i_2},{k_2}}}}}{{{d_{{k_2}}}}}\left\langle {\left[ {\nabla {f_{{k_1}}}({\mathbf{x}}_{{i_1}}^t) - \nabla {f_{{k_1}}}({{{\mathbf{\bar x}}}^t})} \right],\left[ {\nabla {f_{{k_2}}}({\mathbf{x}}_{{i_2}}^t) - \nabla {f_{{k_2}}}({{{\mathbf{\bar x}}}^t})} \right]} \right\rangle } } } }  \nonumber \\
   \leq& \frac{{2{\eta ^2}}}{{{n^2}}}\left( {p - {p^2}} \right)\sum\limits_{i = 1}^n {\sum\limits_{{k_1} = 1}^m {\sum\limits_{{k_2} = 1}^m {\frac{{{s_{i,{k_1}}}}}{{{d_{{k_1}}}}}\frac{{{s_{i,{k_2}}}}}{{{d_{{k_2}}}}}\frac{{{{\left\| {\nabla {f_{{k_1}}}({\mathbf{x}}_i^t) - \nabla {f_{{k_1}}}({{{\mathbf{\bar x}}}^t})} \right\|}^2} + {{\left\| {\nabla {f_{{k_2}}}({\mathbf{x}}_i^t) - \nabla {f_{{k_2}}}({{{\mathbf{\bar x}}}^t})} \right\|}^2}}}{2}} } }  \nonumber \\
  & + \frac{{2{\eta ^2}}}{{{n^2}}}{\left( {1 - p} \right)^2}\sum\limits_{{i_1} = 1}^n {\sum\limits_{{i_2} = 1}^n {\sum\limits_{{k_1} = 1}^m {\sum\limits_{{k_2} = 1}^m {\frac{{{s_{{i_1},{k_1}}}}}{{{d_{{k_1}}}}}\frac{{{s_{{i_2},{k_2}}}}}{{{d_{{k_2}}}}}\frac{{{{\left\| {\nabla {f_{{k_1}}}({\mathbf{x}}_{{i_1}}^t) - \nabla {f_{{k_1}}}({{{\mathbf{\bar x}}}^t})} \right\|}^2} + {{\left\| {\nabla {f_{{k_2}}}({\mathbf{x}}_{{i_2}}^t) - \nabla {f_{{k_2}}}({{{\mathbf{\bar x}}}^t})} \right\|}^2}}}{2}} } } }  \nonumber \\
   \leq &\frac{{2{\eta ^2}}}{{{n^2}}}\left( {p - {p^2}} \right){L^2}\sum\limits_{i = 1}^n {\sum\limits_{{k_1} = 1}^m {\sum\limits_{{k_2} = 1}^m {\frac{{{s_{i,{k_1}}}}}{{{d_{{k_1}}}}}\frac{{{s_{i,{k_2}}}}}{{{d_{{k_2}}}}}{{\left\| {{\mathbf{x}}_i^t - {{{\mathbf{\bar x}}}^t}} \right\|}^2}} } }  \nonumber \\
  & + \frac{{2{\eta ^2}}}{{{n^2}}}{\left( {1 - p} \right)^2}{L^2}\sum\limits_{{i_1} = 1}^n {\sum\limits_{{i_2} = 1}^n {\sum\limits_{{k_1} = 1}^m {\sum\limits_{{k_2} = 1}^m {\frac{{{s_{{i_1},{k_1}}}}}{{{d_{{k_1}}}}}\frac{{{s_{{i_2},{k_2}}}}}{{{d_{{k_2}}}}}\frac{{{{\left\| {{\mathbf{x}}_{{i_1}}^t - {{{\mathbf{\bar x}}}^t}} \right\|}^2} + {{\left\| {{\mathbf{x}}_{{i_2}}^t - {{{\mathbf{\bar x}}}^t}} \right\|}^2}}}{2}} } } }  \nonumber \\
   \leq &\left( {\frac{{2{\eta ^2}}}{{{n^2}}}\left( {p - {p^2}} \right){L^2}a_{\max }^2 + \frac{{2{\eta ^2 m}}}{n^2}{a_{\max }}{{\left( {1 - p} \right)}^2}{L^2}} \right)\left\| {{{\mathbf{X}}^t} - {{{\mathbf{\bar X}}}^t}} \right\|_F^2.
\end{align}
\end{figure*}
we have the bound in (\ref{a11de21}), where $L$-smoothness and the following inequality are applied:
\begin{align}
    \left\langle {{\mathbf{a}},{\mathbf{b}}} \right\rangle  \leq \frac{{{{\left\| {\mathbf{a}} \right\|}^2} + {{\left\| {\mathbf{b}} \right\|}^2}}}{2}, \forall\mathbf{a},\mathbf{b} \in \mathbb{R}^d.
\end{align}

Next, substituting (\ref{a1_1_3}) and (\ref{a11de21}) into (\ref{a11de2}), we can bound the second term in (\ref{a1_1_de}) as 
\begin{align}
    \label{q222}
  \mathbb{E}&\left[ {\left. {{{\left\| {\frac{\eta }{n}\sum\limits_{i = 1}^n {I_i^t} \sum\limits_{k = 1}^m {\frac{1}{{{d_k}}}\nabla {f_k}({\mathbf{x}}_i^t){s_{i,k}}} } \right\|}^2}} \right|{\Im _{t - 1}}} \right] \nonumber \\
   \leq&   \left( {pa_{\max }^2 + m{a_{\max }}\left( {1 - p} \right)} \right)\frac{{2{\eta ^2}}}{{{n^2}}}{L^2}\left( {1 - p} \right)\left\| {{{\mathbf{X}}^t} - {{{\mathbf{\bar X}}}^t}} \right\|_F^2 \nonumber \\
  & + \left[ {\frac{{2{\eta ^2}}}{{{n^3}}}\left( {p - {p^2}} \right){m^2}{L^2} + \frac{{2{\eta ^2}}}{{{n^2}}}{{\left( {1 - p} \right)}^2}{m^2}{L^2}} \right]{\left\| {{{{\mathbf{\bar x}}}^t} - {{\mathbf{x}}^*}} \right\|^2} \nonumber \\
  & + \left[ {\frac{{4{\eta ^2}}}{{{n^2}}}\left( {p - {p^2}} \right){L^2}\sum\limits_{k = 1}^m {\left( {\frac{1}{{{d_k}}} - \frac{1}{n}} \right)} } \right]{\left\| {{{{\mathbf{\bar x}}}^t} - {{\mathbf{x}}^*}} \right\|^2} \nonumber \\
   &+ \frac{{4{\eta ^2}}}{{{n^2}}}\left( {p - {p^2}} \right){C^2}\sum\limits_{k = 1}^m {\left( {\frac{1}{{{d_k}}} - \frac{1}{n}} \right)}.
\end{align}

The third term in (\ref{a1_1_de}) can be expressed as    
    \begin{align}
    \label{a1_1_de3}
    -& \frac{{2\eta }}{n}\mathbb{E}\left[ {\left. {\left\langle {{{{\mathbf{\bar x}}}^t} - {{\mathbf{x}}^*},\sum\limits_{i = 1}^n {I_i^t} \sum\limits_{k = 1}^m {\frac{1}{{{d_k}}}\nabla {f_k}({\mathbf{x}}_i^t){s_{i,k}}} } \right\rangle } \right|{\Im _{t - 1}}} \right] \nonumber \\
   =&  - \frac{{2\eta }}{n}\left( {1 - p} \right)\sum\limits_{i = 1}^n {\left\langle {{{{\mathbf{\bar x}}}^t} - {{\mathbf{x}}^*},\sum\limits_{k = 1}^m {\frac{1}{{{d_k}}}\nabla {f_k}({\mathbf{x}}_i^t){s_{i,k}}} } \right\rangle }  \nonumber \\
   = & - \frac{{2\eta }}{n}\left( {1 - p} \right)\sum\limits_{i = 1}^n {\sum\limits_{k = 1}^m {\frac{{{s_{i,k}}}}{{{d_k}}}\left\langle {{{{\mathbf{\bar x}}}^t} - {\mathbf{x}}_i^t,\nabla {f_k}({\mathbf{x}}_i^t)} \right\rangle } }  \nonumber \\
  & + \frac{{2\eta }}{n}\left( {1 - p} \right)\sum\limits_{i = 1}^n {\sum\limits_{k = 1}^m {\frac{{{s_{i,k}}}}{{{d_k}}}\left\langle {{{\mathbf{x}}^*} - {\mathbf{x}}_i^t,\nabla {f_k}({\mathbf{x}}_i^t)} \right\rangle } }  \nonumber\\
   \leq & \frac{{2\eta }}{n}\left( {1 - p} \right)\sum\limits_{k = 1}^m {\left[ {{f_k}\left( {{{\mathbf{x}}^*}} \right) - {f_k}\left( {{{{\mathbf{\bar x}}}^t}} \right)} \right]}  \nonumber \\
& + \frac{{\eta L}}{n}\left( {1 - p} \right)\sum\limits_{i = 1}^n {{{\left\| {{{{\mathbf{\bar x}}}^t} - {\mathbf{x}}_i^t} \right\|}^2}{a_i}} \nonumber\\
& - \frac{{\mu \eta }}{n}\left( {1 - p} \right)\sum\limits_{i = 1}^n {{{\left\| {{{\mathbf{x}}^*} - {\mathbf{x}}_i^t} \right\|}^2}{a_i}}  \nonumber\\
   \leq & \frac{{2\eta }}{n}\left( {1 - p} \right)\sum\limits_{k = 1}^m {\left[ {{f_k}\left( {{{\mathbf{x}}^*}} \right) - {f_k}\left( {{{{\mathbf{\bar x}}}^t}} \right)} \right]}  \nonumber \\
& + \frac{{\eta L}}{n}\left( {1 - p} \right){a_{\max }}\sum\limits_{i = 1}^n {{{\left\| {{{{\mathbf{\bar x}}}^t} - {\mathbf{x}}_i^t} \right\|}^2}}  \nonumber \\
 &- \frac{{\mu \eta }}{n}\left( {1 - p} \right){a_{\min }}\sum\limits_{i = 1}^n {{{\left\| {{{\mathbf{x}}^*} - {\mathbf{x}}_i^t} \right\|}^2}}  \nonumber \\
   \leq & \frac{{2\eta }}{n}\left( {1 - p} \right)\sum\limits_{k = 1}^m {\left[ {{f_k}\left( {{{\mathbf{x}}^*}} \right) - {f_k}\left( {{{{\mathbf{\bar x}}}^t}} \right)} \right]} \nonumber\\
& + \frac{{\eta L}}{n}\left( {1 - p} \right){a_{\max }}\left\| {{{{\mathbf{\bar X}}}^t} - {{\mathbf{X}}^t}} \right\|_F^2 \nonumber\\
  & - \frac{{\mu \eta }}{n}\left( {1 - p} \right){a_{\min }}\sum\limits_{i = 1}^n {\left[ {\frac{1}{2}{{\left\| {{{{\mathbf{\bar x}}}^t} - {{\mathbf{x}}^*}} \right\|}^2} - {{\left\| {{{{\mathbf{\bar x}}}^t} - {\mathbf{x}}_i^t} \right\|}^2}} \right]} \nonumber \\
   = & 2\eta \frac{m}{n} \left( {1 - p} \right)\left[ {f\left( {{{\mathbf{x}}^*}} \right) - f\left( {{{{\mathbf{\bar x}}}^t}} \right)} \right] \nonumber \\
 & + \left[ {\frac{{\eta L}}{n}\left( {1 - p} \right){a_{\max }} + \frac{{\mu \eta }}{{n}}\left( {1 - p} \right){a_{\min }}} \right]\left\| {{{{\mathbf{\bar X}}}^t} - {{\mathbf{X}}^t}} \right\|_F^2 \nonumber \\
 & - \frac{1}{2}\mu \eta \left( {1 - p} \right){a_{\min }}{\left\| {{{{\mathbf{\bar x}}}^t} - {{\mathbf{x}}^*}} \right\|^2},
\end{align} 
where $L$-smoothness and $\mu$-strongly convexity and the inequality in (\ref{ineq2}) are utilized.  
Substituting (\ref{q222}) and (\ref{a1_1_de3}) into (\ref{a1_1_de}), we can obtain that 
    \begin{align}
    \label{q12}
  \mathbb{E}&\left[ {\left. {{{\left\| {{{{\mathbf{\bar x}}}^t} - {{\mathbf{x}}^*} - \frac{\eta }{n}\sum\limits_{i = 1}^n {I_i^t} \sum\limits_{k = 1}^m {\frac{1}{{{d_k}}}\nabla {f_k}({\mathbf{x}}_i^t){s_{i,k}}} } \right\|}^2}} \right|{\Im _{t - 1}}} \right] \nonumber \\
   \leq &{q_1}{\left\| {{{{\mathbf{\bar x}}}^t} - {{\mathbf{x}}^*}} \right\|^2} \nonumber \\
  &+ {q_2}\left\| {{{{\mathbf{\bar X}}}^t} - {{\mathbf{X}}^t}} \right\|_F^2{\text{ + }}\frac{{4{\eta ^2}}}{{{n^2}}}\left( {p - {p^2}} \right){C^2}\sum\limits_{k = 1}^m {\left( {\frac{1}{{{d_k}}} - \frac{1}{n}} \right)}  \nonumber \\
  & + 2\eta \frac{m}{n} \left( {1 - p} \right)\left[ {f\left( {{{\mathbf{x}}^*}} \right) - f\left( {{{{\mathbf{\bar x}}}^t}} \right)} \right].
\end{align}  

Next, substituting (\ref{a1_1_2}) and (\ref{q12}) into (\ref{a1_1}) yields Lemma \ref{xt-yitaIG-x*}.

 \section{Proof of Lemma \ref{GOCO-devia-optimal}}
  \label{appendix GOCO-devia-optimal}
For GOCO, for \(\forall \Delta  > 0\), according to (\ref{x1/2}) and (\ref{gossip_goco}), we can derive
\begin{align}   
\label{goco-1}
\mathbb{E}&\left[ {\left. {{{\left\| {{{{\mathbf{\bar x}}}^{t + 1}} - {{\mathbf{x}}^*}} \right\|}^2}} \right|{\Im _{t - 1}}} \right]\nonumber \\
    =& \mathbb{E}\left[ {\left. {{{\left\| {{{{\mathbf{\bar x}}}^{t + \frac{1}{2}}} + \gamma \left( {\left( {{{{\mathbf{X}}}^{t}}({\mathbf{W}} - {\mathbf{I}})} \right) \odot {{\mathbf{B}}^t}} \right)\frac{{\mathbf{1}}}{n} - {{\mathbf{x}}^*}} \right\|}^2}} \right|{\Im _{t - 1}}} \right] \nonumber \\
   = &\mathbb{E}\left[ {\left\| {{{{\mathbf{\bar x}}}^t} - \eta \left( {{{\mathbf{G}}^t} \odot {{\mathbf{B}}^t}} \right)\frac{{\mathbf{1}}}{n}} \right.} \right. \nonumber \\
 & \left. {\left. {{{\left. { + \gamma \left( {\left( {{{{\mathbf{X}}}^{t}}({\mathbf{W}} - {\mathbf{I}})} \right) \odot {{\mathbf{B}}^t}} \right)\frac{{\mathbf{1}}}{n} - {{\mathbf{x}}^*}} \right\|}^2}} \right|{\Im _{t - 1}}} \right] \nonumber \\
   = &\mathbb{E}\left[ {\left\| {{{{\mathbf{\bar x}}}^t} -   \frac{\eta }{n}\sum\limits_{i = 1}^n {I_i^t} {\mathbf{g}}_i^t - {{\mathbf{x}}^*}} \right.} \right. \nonumber \\
 & \left. {\left. {{{\left. { + \gamma \left( {\left( {{{{\mathbf{X}}}^{t}}({\mathbf{W}} - {\mathbf{I}})} \right) \odot {{\mathbf{B}}^t}} \right)\frac{{\mathbf{1}}}{n}} \right\|}^2}} \right|{\Im _{t - 1}}} \right] \nonumber \\
   \leq &\left( {1 + \Delta } \right)\mathbb{E}\left[ {\left. {{{\left\| {{{{\mathbf{\bar x}}}^t} - \frac{\eta }{n}\sum\limits_{i = 1}^n {I_i^t} {\mathbf{g}}_i^t - {{\mathbf{x}}^*}} \right\|}^2}} \right|{\Im _{t - 1}}} \right] \nonumber\\
  & + \left( {1 + \frac{1}{\Delta }} \right)\mathbb{E}\left[ {\left. {{{\left\| {\gamma \left( {\left( {{{{\mathbf{ X}}}^{t}}({\mathbf{W}} - {\mathbf{I}})} \right) \odot {{\mathbf{B}}^t}} \right)\frac{{\mathbf{1}}}{n}} \right\|}^2}} \right|{\Im _{t - 1}}} \right],
\end{align}
 where the following inequality is applied:
\begin{align}
    \label{ineq1}
    \left\| \mathbf{a} + \mathbf{b} \right\|^2 &\leq (1 + \alpha) \left\| \mathbf{a} \right\|^2 + (1 + \alpha^{-1}) \left\| \mathbf{b} \right\|^2, \nonumber\\
  &  \mathbf{a}, \mathbf{b} \in \mathbb{R}^d
, \alpha>0.
\end{align}
and
\begin{align}
\label{alg24222}
{{\mathbf{X}}^{t + \frac{1}{2}}} =& \left[ {{\mathbf{x}}_1^{t + \frac{1}{2}},...,{\mathbf{x}}_n^{t + \frac{1}{2}}} \right].
\end{align}

In (\ref{goco-1}), based on  (\ref{Iit_bernoulli}), for the second term, we have 
\begin{align}
    \label{goco-2}
  \mathbb{E}&\left[ {\left. {{{\left\| {\gamma \left( {\left( {{{{\mathbf{X}}}^{t}}\left( {{\mathbf{W}} - {\mathbf{I}}} \right)} \right) \odot {{\mathbf{B}}^t}} \right)\frac{{\mathbf{1}}}{n}} \right\|}^2}} \right|{\Im _{t - 1}}} \right] \nonumber \\
   = &\mathbb{E}\left[ {\left. {{{\left\| {\gamma \left( {\left( {\left( {{{{\mathbf{X}}}^{t}} - {{{\mathbf{\bar X}}}^{t}}} \right)\left( {{\mathbf{W}} - {\mathbf{I}}} \right)} \right) \odot {{\mathbf{B}}^t}} \right)\frac{{\mathbf{1}}}{n}} \right\|}^2}} \right|{\Im _{t - 1}}} \right] \nonumber \\
   =& \frac{{{\gamma ^2}}}{{{n^2}}}\mathbb{E}\left[ {\left. {{{\left\| {\sum\limits_{i = 1}^n {I_i^t{{\left( {\left( {{{{\mathbf{X}}}^{t}} - {{{\mathbf{\bar X}}}^{t}}} \right)\left( {{\mathbf{W}} - {\mathbf{I}}} \right)} \right)}_i}} } \right\|}^2}} \right|{\Im _{t - 1}}} \right] \nonumber \\
   \leq& \frac{{{\gamma ^2}}}{n}\sum\limits_{i = 1}^n {\mathbb{E}\left[ {\left. {I_i^t{{\left\| {{{\left( {\left( {{{{\mathbf{X}}}^{t}} - {{{\mathbf{\bar X}}}^{t}}} \right)\left( {{\mathbf{W}} - {\mathbf{I}}} \right)} \right)}_i}} \right\|}^2}} \right|{\Im _{t - 1}}} \right]}  \nonumber \\
   =& \frac{{{\gamma ^2}}}{n}\left( {1 - p} \right)\left\| {\left( {{{{\mathbf{X}}}^{t}} - {{{\mathbf{\bar X}}}^{t}}} \right)\left( {{\mathbf{W}} - {\mathbf{I}}} \right)} \right\|_F^2 \nonumber \\
   \leq &\frac{{{\gamma ^2}}}{n}\left( {1 - p} \right)\left\| {{{{\mathbf{X}}}^{t}} - {{{\mathbf{\bar X}}}^{t}}} \right\|_F^2\left\| {{\mathbf{W}} - {\mathbf{I}}} \right\|_2^2 \nonumber \\
   = &\frac{{{\gamma ^2}}}{n}\left( {1 - p} \right){\beta ^2}\left\| {{{{\mathbf{X}}}^{t}} - {{{\mathbf{\bar X}}}^{t}}} \right\|_F^2,
\end{align}
where ${{{\left( {\left( {{{{\mathbf{X}}}^{t}} - {{{\mathbf{\bar X}}}^{t}}} \right)\left( {{\mathbf{W}} - {\mathbf{I}}} \right)} \right)}_i}}$ denotes the \(i{\text{ - th}}\) column in ${\left( {{{{\mathbf{X}}}^{t}} - {{{\mathbf{\bar X}}}^{t}}} \right)\left( {{\mathbf{W}} - {\mathbf{I}}} \right)}$, ${{\mathbf{\bar X}}^{t}}\left( {{\mathbf{W}} - {\mathbf{I}}} \right) = {\mathbf{0}}$ is noted and we have used (\ref{ineq2}) and the following basic inequality:
\begin{align}
       \label{ineq3}
    {\left\| {{\mathbf{AB}}} \right\|_F} \leq {\left\| {\mathbf{A}} \right\|_F}{\left\| {\mathbf{B}} \right\|_2},{\mathbf{A}} \in {\mathbb{R}^{d \times n}}&,{\mathbf{B}} \in  {\mathbb{R}^{n \times n}}.
\end{align}

In (\ref{goco-1}), for the first term, we have already derived the bound in Lemma \ref{xt-yitaIG-x*}. Substituting (\ref{a11_final}) and (\ref{goco-2}) into (\ref{goco-1}) yields Lemma \ref{GOCO-devia-optimal}.

\section{Proof of Lemma \ref{GOCO-divia1}}
  \label{appendix GOCO-divia1}
  According to (\ref{x1/2}) and (\ref{gossip_goco}), we can derive that 
  \begin{align}
      \label{equ-goco-divia1}
  \mathbb{E}&\left[ {\left. {\left\| {{{\mathbf{X}}^{t + 1}} - {{{\mathbf{\bar X}}}^{t + 1}}} \right\|_F^2} \right|{\Im _{t - 1}}} \right] \nonumber \\
   \leq& \mathbb{E}\left[ {\left. {\left\| {{{\mathbf{X}}^{t + 1}} - {{{\mathbf{\bar X}}}^t}} \right\|_F^2} \right|{\Im _{t - 1}}} \right] \nonumber \\
   =& \mathbb{E}\left[ {\left. {\left\| {{{\mathbf{X}}^{t + \frac{1}{2}}} - {{{\mathbf{\bar X}}}^t} + \gamma \left( {\left( {{{\mathbf{X}}^t}\left( {{\mathbf{W}} - {\mathbf{I}}} \right)} \right) \odot {{\mathbf{B}}^t}} \right)} \right\|_F^2} \right|{\Im _{t - 1}}} \right] \nonumber \\
   = &\mathbb{E}\left[ {\left. {\left\| {{{\mathbf{X}}^t} - \eta {{\mathbf{G}}^t} \odot {{\mathbf{B}}^t} - {{{\mathbf{\bar X}}}^t} + \gamma \left( {\left( {{{\mathbf{X}}^t}\left( {{\mathbf{W}} - {\mathbf{I}}} \right)} \right) \odot {{\mathbf{B}}^t}} \right)} \right\|_F^2} \right|{\Im _{t - 1}}} \right] \nonumber \\
   =& \left( {1 - p} \right)\mathbb{E}\left[ {\left. {\left\| {{{\mathbf{X}}^t} - {{{\mathbf{\bar X}}}^t} - \eta {{\mathbf{G}}^t} + \gamma {{\mathbf{X}}^t}\left( {{\mathbf{W}} - {\mathbf{I}}} \right)} \right\|_F^2} \right|{\Im _{t - 1}}} \right] \nonumber\\
   &+ p\mathbb{E}\left[ {\left. {\left\| {{{\mathbf{X}}^t} - {{{\mathbf{\bar X}}}^t}} \right\|_F^2} \right|{\Im _{t - 1}}} \right] \nonumber \\
   \leq &\left( {1 - p} \right)\left( {1 + \varpi } \right)\mathbb{E}\left[ {\left. {\left\| {{{\mathbf{X}}^t} - {{{\mathbf{\bar X}}}^t} + \gamma {{\mathbf{X}}^t}\left( {{\mathbf{W}} - {\mathbf{I}}} \right)} \right\|_F^2} \right|{\Im _{t - 1}}} \right]\nonumber\\
   &+ \left( {1 - p} \right)\left( {1 + \frac{1}{\varpi }} \right)\mathbb{E}\left[ {\left. {\left\| {\eta {{\mathbf{G}}^t}} \right\|_F^2} \right|{\Im _{t - 1}}} \right] \nonumber\\
   &+ p\mathbb{E}\left[ {\left. {\left\| {{{\mathbf{X}}^t} - {{{\mathbf{\bar X}}}^t}} \right\|_F^2} \right|{\Im _{t - 1}}} \right] \nonumber \\
   \leq & \left( {1 - p} \right)\left( {1 + \varpi } \right){\left( {1 - \rho \gamma } \right)^2}\mathbb{E}\left[ {\left. {\left\| {{{\mathbf{X}}^t} - {{{\mathbf{\bar X}}}^t}} \right\|_F^2} \right|{\Im _{t - 1}}} \right]\nonumber\\
   &+ p\mathbb{E}\left[ {\left. {\left\| {{{\mathbf{X}}^t} - {{{\mathbf{\bar X}}}^t}} \right\|_F^2} \right|{\Im _{t - 1}}} \right] + \left( {1 - p} \right)\left( {1 + \frac{1}{\varpi }} \right){\eta ^2}{G^2}w_2, \nonumber\\
  & \forall \varpi  > 0,
  \end{align}
  based on (\ref{ineq1}) and Lemma 17 in \cite{koloskova2019decentralized}, and the following bound:  
  \begin{align}
      \label{lemma34}
  \mathbb{E}&\left[ {\left. {\left\| {{{\mathbf{G}}^t}} \right\|_F^2} \right|{\Im _{t - 1}}} \right] \nonumber \\
  =&\sum\limits_{i = 1}^n {\mathbb{E}\left[ {\left. {{{\left\| {\sum\limits_{k = 1}^m {\frac{1}{{{d_k}}}\nabla {F_k}({\mathbf{x}}_i^t,\xi _{i,k}^t){s_{i,k}}} } \right\|}^2}} \right|{\Im _{t - 1}}} \right]}  \nonumber \\
   \leq& {G^2}\sum\limits_{i = 1}^n {\sum\limits_{{k_1} = 1}^m {\sum\limits_{k = 1}^m {\frac{{{s_{i,{k_1}}}{s_{i,k}}}}{{d_k^2}}} } }  \nonumber \\
   =& {G^2}\left[ {\sum\limits_{k = 1}^m {\frac{1}{{{d_k}}} + } \frac{1}{n}\left( {\sum\limits_{k = 1}^m {\frac{1}{{{d_k}}}} } \right)\left( {\sum\limits_{k = 1}^m {{d_k}} } \right) - \frac{m}{n}} \right].
  \end{align}
 Here, the first equality is obtained from (\ref{g}), the inequality is derived based on (\ref{boundvar2}) and (\ref{ineq2}), and the last equality is yielded by 
  \begin{align}
      \label{lemma35}
 & \sum\limits_{i = 1}^n {\sum\limits_{{k_1} = 1}^m {\sum\limits_{k = 1}^m {\frac{{{s_{i,{k_1}}}{s_{i,k}}}}{{d_k^2}}} } }  \nonumber \\
  & = \sum\limits_{k = 1}^m {\frac{{\sum\limits_{i = 1}^n {{s_{i,k}}} }}{{d_k^2}} + } \sum\limits_{{k_1} = 1}^m {\sum\limits_{k \ne {k_1}}^m {\frac{{\sum\limits_{i = 1}^n {{s_{i,{k_1}}}{s_{i,k}}} }}{{d_k^2}}} }   \nonumber \\
  & = \sum\limits_{k = 1}^m {\frac{1}{{{d_k}}} + } \sum\limits_{{k_1} = 1}^m {\sum\limits_{k \ne {k_1}}^m {\frac{{{d_{{k_1}}}}}{{n{d_k}}}} }  \nonumber \\
  & = \sum\limits_{k = 1}^m {\frac{1}{{{d_k}}} + } \sum\limits_{{k_1} = 1}^m {\sum\limits_{k = 1}^m {\frac{{{d_{{k_1}}}}}{{n{d_k}}}} }  - \sum\limits_{k = 1}^m {\frac{1}{n}}  \nonumber \\
  & = \sum\limits_{k = 1}^m {\frac{1}{{{d_k}}} + } \frac{1}{n}\left( {\sum\limits_{k = 1}^m {\frac{1}{{{d_k}}}} } \right)\left( {\sum\limits_{k = 1}^m {{d_k}} } \right) - \frac{m}{n},
  \end{align}
  based on the pair-wise balanced manner of data distribution. 
  
  By setting \(\varpi  = \frac{{\gamma \rho }}{2}\) in (\ref{equ-goco-divia1}), Lemma \ref{GOCO-divia1} is proved.

\end{document}